\documentclass[conference]{IEEEtran}
\IEEEoverridecommandlockouts

\usepackage{cite}
\usepackage{amsmath,amssymb,amsfonts}
\usepackage{algpseudocode}
\usepackage{algorithm}
\usepackage{graphicx}
\usepackage{textcomp}
\usepackage{xcolor}
\def\BibTeX{{\rm B\kern-.05em{\sc i\kern-.025em b}\kern-.08em
    T\kern-.1667em\lower.7ex\hbox{E}\kern-.125emX}}

\newcommand{\bigtimes}{\mathop{\vcenter{\hbox{\scalebox{1.5}{$\times$}}}}}

\newcommand{\QED}{\hfill\ensuremath{\blacksquare}}

\usepackage{booktabs}
\usepackage{caption}
\usepackage{subcaption}
\usepackage{multirow}
\usepackage{hyperref}
\usepackage{tikz}
\usetikzlibrary{trees,arrows.meta,positioning}

\newcommand{\U}{\mathcal{U}}
\newcommand{\R}{\mathbb{R}}
\newcommand{\E}{\mathbb{E}}
\newcommand{\1}{\mathbf{1}}

\newcommand{\ubar}[1]{\overline{#1}}
\newtheorem{proposition}{Proposition}

\newtheorem{example}{Example}

\newtheorem{remark}{Remark}
\newtheorem{proof}{Proof}

\renewcommand{\vec}[1]{\boldsymbol{#1}}

\begin{document}

\title{Interval-valued SHAP in Tree-Based Models}



\author{
\IEEEauthorblockN{
Chenrui Zhu\IEEEauthorrefmark{1},
Vu-Linh Nguyen\IEEEauthorrefmark{1},
Marie-Hélène Masson\IEEEauthorrefmark{1}\IEEEauthorrefmark{2},
Sébastien Destercke\IEEEauthorrefmark{1}
}

\IEEEauthorblockA{
\IEEEauthorrefmark{1}
\textit{Université de technologie de Compiègne, UMR-CNRS 7253 Heudiasyc}\\
Compiègne, France
}

\IEEEauthorblockA{
\IEEEauthorrefmark{2}
\textit{IUT de l’Oise, Université de Picardie Jules Verne}\\
Beauvais, France\\
\{chenrui.zhu, vu-linh.nguyen, mylene.masson, sebastien.destercke\}@hds.utc.fr
}
}

\maketitle

\begingroup
\renewcommand{\thefootnote}{}
\footnotetext{}
\endgroup

\begingroup
\renewcommand{\thefootnote}{}
\footnotetext{%
\textit{Accepted for publication in the Proceedings of the 2026 IEEE International Conference on Data Mining (ICDM 2026).} 
\copyright\ 2026 IEEE.  Personal use of this material is permitted.  Permission from IEEE must be obtained for all other uses, in any current or future media, including reprinting/republishing this material for advertising or promotional purposes, creating new collective works, for resale or redistribution to servers or lists, or reuse of any copyrighted component of this work in other works.
}
\endgroup

\begin{abstract}

Shapley values are among the most popular feature-attribution explanations. Efficient approaches for computing/estimating Shapley values for tree-based models, which are state-of-the-art for tabular data sets, have been developed. However, it is known that Shapley values can be (highly) unrobust due to small and realistic changes. In this paper, we propose an imprecise Dirichlet model (IDM) based method to analyze the robustness of Shapley values in decision trees and random forests. Technically, it is done by quantifying and analyzing the interval-valued Shapley values when a few unannotated instances are randomly introduced to the leaves of the trees. The interval-valued Shapley values can be defined following common principles in handling incomplete data: the pessimistic and averaging principles. We derive various theoretical results that lead to efficient computation of the interval-valued Shapley values. We also show that the proposed method can be straightforwardly generalized to the case of Banzhaf values. We then present various case studies and experiments to illustrate the behaviour of the proposed interval-valued Shapley values and their applications in debiasing uninformative features.

\end{abstract}


\section{Introduction}\label{sec:introduction}



Explanation methods aim to make predictive models more transparent by providing human-interpretable summaries of their behavior. Feature-attribution explanations are among the most popular explanations. For each instance, they assign to each feature a real-valued score quantifying its attribution, i.e., its contribution to the prediction given by the predictive model. The most popular feature-attribution explanations are SHapley Additive exPlanations (SHAP), whose attributions are characterized by axioms such as local accuracy and consistency \cite{lundberg2018consistent,10.5555/3295222.3295230}. Efficient approaches, notably TreeSHAP \cite{lundberg2018consistent,lundberg2019explainable}, for computing Shapley values for tree-based models, which are state-of-the-art for tabular data sets \cite{grinsztajn2022tree}, have been developed. This has led to the broad adoption of Shapley-style explanations for tree-based models, notably gradient-boosted trees and random forests~\cite{lundberg2018consistent,lundberg2019explainable}.

However, it is known that Shapley values can be fragile, e.g., small, realistic input perturbations can cause large shifts in the attributions \cite{kasbohrer2026assessing}. Their fragility would be further amplified in the case of tree-based models due to the fact that small perturbations to the training data can significantly change the trees  \cite{last2002improving}, and thus the attributions. These have led to various attempts to assess the robustness of Shapley values for tree-based models \cite{kasbohrer2026assessing,utkin2021imprecise}. A notable attempt is the imprecise SHAP \cite{utkin2021imprecise}, developed for imprecise-probability classifiers. The approach starts by assuming that the classifiers predict credal sets and employs the KL divergence (instead of the difference) to define the Shapley values. Imprecise outputs of the classifiers naturally lead to interval-valued Shapley values, which essentially reflect the uncertainty (due to insufficient training data). While this is a generic approach with interesting theoretical properties, it is unclear how to leverage it to assess the robustness of Shapley values for tree-based models. In particular, due to differences in their notions of Shapley values, it is unclear under which conditions imprecise Shapley values can cover the Shapley values of tree-based models, which might be a crucial requirement for robustness indicators.      

In this paper, we focus on assessing the robustness of Shapley values when the stability of the tree-based models is of utmost importance, notably for the purpose of interpretability. Note that in safety-critical applications, such as in the medical domain \cite{banerjee2019tree,xu2023interpretability}, understanding the structure of the trees is important. In other words, it might be advisable to guarantee the model stability unless additional data coming at the deployment phase suggests the need for model retraining, e.g., due to unsatisfactory generalization ability. This is because understanding the structure of the trees can be costly, but it can enhance the trustworthiness of the decision support system. Model stability in random forests is also important for interpretability, notably when the structures of its trees are leveraged to extract rules~\cite{tetteh2025optimization}. 

Note that with fixed structures, the Shapley values for some individuals can also be very fragile due to small perturbations to training data. This will be confirmed in our empirical study. This naturally calls for methods for assessing the Shapley values for the individuals, which should compactly present, for each instance, the possible changes in all the attributions. To serve this purpose, we propose an imprecise Dirichlet model (IDM) ~\cite{BERNARD2005123} based method to analyze the robustness of Shapley values in decision trees and random forests. The overall idea is to associate the robustness of each Shapley value with an interval-valued Shapley value, which covers the original value, induced by introducing a small number of unannotated instances into the training data. This can be implemented following common principles in handling incomplete data: the pessimistic and averaging principles, see, e.g., \cite{do2026probabilistic} and references therein. The corresponding problems of computing interval-valued Shapley values are combinatorial optimization problems, where Brute-force search is impractical.    

We present various theoretical results that allow efficient computation of interval-valued Shapley values in the case of decision trees. For each instance, we show that the Shapley value associated with each feature can be expressed as a linear combination of the leaf values. 
The efficiency of the averaging principle further calls for efficiency in computing expected values of multinomial distributions. The efficiency of the pessimistic principle is achieved by leveraging the linearity of the Shapley values and the property that, for each leaf, the lower (upper) leaf value provided by the IDM model is a convex decreasing (concave increasing) function in the number of unannotated instances. We then show that the mentioned results can be extended to the case of Banzhaf values straightforwardly. Finally, we extend the results achieved for decision trees to random forests.    

After recalling the basics of Shapley values and tree-based models in Section~\ref{sec:background}, Section~\ref{sec:leaf-perturbation} presents our IDM-based method for deriving interval-valued Shapley values. 
Section~\ref{sec:experiments} presents various case studies and experiments to illustrate the behaviour of the proposed interval-valued Shapley values and their applications in debiasing uninformative features \cite{loecher2024debiasing}. Finally, Section \ref{sec:conclusion} concludes this work and sketches out potential follow-up work.

\section{Background}\label{sec:background}
We consider decision-tree predictors and game-theoretic feature-attribution explanations. $\mathcal{U}:=\{1,\dots,n\}$ denotes the set of feature indices and
$\mathcal{X}\subseteq\R^n$ the input space. We fix an instance $\vec{x} = (x_1, \ldots, x_n) \in\mathcal{X}$ to be explained.

\subsection{Decision-tree models}\label{subsec:background-trees}

A binary decision tree is a rooted directed tree whose internal nodes perform threshold tests on single features. Each internal node \(v\) is labeled by a split feature \(d(v)\in\mathcal U\) and a threshold \(th(v)\in\mathbb R\). Given an input \(\vec{x}\in\mathcal X\), the traversal follows the left child if \(x_{d(v)}<th(v)\) and the right child otherwise, until a leaf is reached. Let \(\mathcal L\) denote the set of leaves of the tree trained on the training data.

Throughout the paper, we distinguish the \emph{tree structure} \(T\), consisting of the topology together with the split variables and thresholds, from the \emph{leaf outputs}. The structure is kept fixed, while the leaf outputs will later be allowed to vary. Each leaf \(\ell\in\mathcal L\) carries a scalar value
\(w_\ell\in\mathbb R\), such as a regression mean, a target-class probability, or a
logit. We collect these values in the vector \(w := (w_\ell)_{\ell\in\mathcal L}\in\mathbb R^{|\mathcal L|}.\) In our case, as we are concerned with supervised learning, we will focus on target-class probability.

Let \(\ell_T(\vec{x})\in\mathcal L\) be the leaf reached by \(\vec{x}\), and let \(R_\ell := \{\vec{x}\in\mathcal X:\ell_T(\vec{x})=\ell\}\) be the set of all inputs that would be routed to leaf \(\ell\). Then the prediction function admits the leaf-sum representation
\begin{equation}\label{eq:tree-leaf-sum-bg}
f_{T,w}(\vec{x})
=
w_{\ell_T(\vec{x})}
=
\sum_{\ell\in\mathcal L} w_\ell\,\mathbf 1\{\vec{x}\in R_\ell\}.
\end{equation}
i.e., the output is the value attached to the leaf selected by traversal \cite{molnar2022interpretable}.

\subsection{Feature attribution as a coalitional game}\label{subsec:game}

Local feature-attribution methods such as SHAP can be formulated as cooperative games over features \cite{10.5555/3295222.3295230,lundberg2018consistent,aas2021dependent}. Let \(\mathbf{X}=(X_1,\dots,X_n)\sim D\) be an \(\mathcal X\)-valued background random vector. For each subset \(S\subseteq\mathcal U\), we denote by \(\bar S:=\mathcal U\setminus S\). We use the interventional (marginal) value function
\begin{equation}\label{eq:interventional-game-bg}
v_w(S)
:=
\E_{\mathbf{X}_{\bar S}}\!\left[
f_{T,w}\big(\vec{x}_S,\mathbf{X}_{\bar S}\big)
\right],
\end{equation}
where the hybrid vector \((\vec{x}_S,\mathbf{X}_{\bar S})\) fixes the coordinates in \(S\) to their values in the explained instance and samples the remaining coordinates from the background distribution \(D\). Equivalently, this semantics is often written as \(v_w(S)=\E_{\mathbf{X}_{\bar S}}[f_{T,w}(\mathbf{X})\mid \mathrm{do}(\mathbf{X}_S=\vec{x}_S)],\) emphasizing that the features in \(S\) are intervened on rather than conditioned on.

The baseline is \(v_w(\emptyset)=f_{T,w}(\vec{x})\). Other missing-feature semantics can also be used, including conditional expectations for dependent features \cite{aas2021dependent} and path-dependent semantics for tree models \cite{lundberg2018consistent,pmlr-v180-karczmarz22a}. Our results apply to any
semantics for which \(v_w(S)\) can be written as an expectation of \(f_{T,w}\) under a distribution that does not depend on the leaf-output vector~\(w\).

\subsection{Shapley attributions}

Given a set function $v:2^\U\to\R$, the Shapley attribution of feature \(i\in\mathcal U\) is:
\begin{equation}\label{eq:shapley-bg}
\phi_i
\;:=\;
\sum_{S\subseteq \U\setminus\{i\}}
\frac{|S|!\,(n-|S|-1)!}{n!}\,
\big(v(S\cup\{i\})-v(S)\big).
\end{equation}

The Shapley value is uniquely characterized by axioms \cite{molnar2022interpretable} such as efficiency, symmetry, linearity, and dummy. A standard convention defines an intercept $\phi_0:=v(\emptyset)$ so that the \emph{efficiency} identity reads
\begin{equation}\label{eq:efficiency-bg}
\begin{aligned}
& \phi_0 + \sum_{i\in\U}\phi_i \;=\; v(\U)\;=\; f_{T,w}(\vec{x}),
\\ \text{equivalently} & \qquad \sum_{i\in\U}\phi_i \;=\; v(\U)-v(\emptyset).
\end{aligned}
\end{equation}

\begin{remark} 
An alternative game-theoretic attribution is the raw Banzhaf value, which averages marginal contributions uniformly 
\[
\beta_i
:=
\frac{1}{2^{n-1}}
\sum_{S\subseteq \U\setminus\{i\}}
\big(v(S\cup\{i\})-v(S)\big).
\]
Unlike the Shapley value, the raw Banzhaf value generally does not satisfy the efficiency identity \(\sum_{i\in\mathcal U}\beta_i = v(\mathcal U)-v(\emptyset).\) Nevertheless, it has attractive computational properties for tree models and often produces feature rankings similar to Shapley attributions in practice
\cite{pmlr-v180-karczmarz22a}.  
\end{remark}

\subsection{Leaf-linear representation of the game and attributions}\label{sec:leaf-linearity}
Recall from Section~\ref{sec:background} that a fixed tree structure $T$ together with leaf values
$w=(w_\ell)_{\ell\in\mathcal{L}}$ defines the predictor $f_{T,w}$. Fixing $T$, an instance $\vec{x}\in\mathcal{X}$ and a background distribution $D$ and the missing-feature semantics, for each coalition $S\subseteq\mathcal{U}$, the induced game value $v_w(S)$ depends on $w$ through the leaf outputs. This section shows that $v_w(S)$ is linear in $w$, and consequently that the attribution scores are linear forms in $w$.

\subsubsection{Linearity of the coalition value $v_w(S)$}
For a coalition $S\subseteq\U$, recall the associated game value
\begin{equation}\label{eq:vw-def}
v_w(S)\;:=\;\E_{\mathbf{X}_{\bar S}}\!\left[f_{T,w}\big(\vec{x}_S, \mathbf{X}_{\ubar S}\big)\right].
\end{equation}
This is the expected prediction of the tree when the features in $S$ are fixed at their values in the explained instance $\vec{x}$, while the remaining features $\ubar S:=\U\setminus S$ are integrated out according to the background distribution.
Hence, $v_w(S)$ quantifies the predictive contribution of knowing the features in $S$. Since the tree prediction is piecewise constant over leaves, $v_w(S)$ can be written in terms of the probabilities of reaching each leaf. For every leaf $\ell\in\mathcal L$, define the \emph{reach probability}
\begin{align}
q_\ell(S) &\;:=\;\mathbb P\Big(\big(\vec{x}_S,\mathbf{X}_{\ubar S}\big)\in R_\ell\Big)
\; \nonumber\\
&=\;
\E_{\mathbf{X}_{\bar S}}\!\left[\1\Big\{\big(\vec{x}_S,\mathbf{X}_{\ubar S}\big)\in R_\ell\Big\}\right]. \label{eq:q-def}
\end{align}
Let $q(S):=(q_\ell(S))_{\ell\in\mathcal L}\in[0,1]^{|\mathcal L|}$. The quantity $q_\ell(S)$ is the probability that the hybrid point $\big(\vec{x}_S,\mathbf{X}_{\ubar S}\big)$ falls into leaf $\ell$. It depends on the tree structure, the explained instance $\vec{x}$, the coalition $S$, the background distribution and the missing-feature semantics, but crucially it does \emph{not} depend on the leaf values $w$. Thus $q(S)$ captures how the coalition $S$ distributes probability mass over the leaves, whereas $w$ determines the prediction associated with each leaf.

\begin{proposition}\label{prop:vw-linear}
For every coalition $S\subseteq\U$,
\begin{equation}\label{eq:vw-linear}
v_w(S)\;=\;\sum_{\ell\in\mathcal L} q_\ell(S)\,w_\ell
\;=\;\langle q(S),w\rangle.
\end{equation}
In particular, for each fixed $S$, the map $w\mapsto v_w(S)$ is linear.
\end{proposition}

\begin{proof}
Insert the leaf decomposition of $f_{T,w}$ into \eqref{eq:vw-def} and exchange the
finite sum and expectation:
\begin{align*}
v_w(S)
=
\E_{\mathbf{X}_{\ubar S}}\!\left[\sum_{\ell\in\mathcal L} w_\ell\,
\1\Big\{\big(\vec{x}_S,\mathbf{X}_{\ubar S}\big)\in R_\ell\Big\}\right] 
\\
=
\sum_{\ell\in\mathcal L} w_\ell\,
\E_{\mathbf{X}_{\ubar S}}\!\left[\1\Big\{\big(\vec{x}_S,\mathbf{X}_{\ubar S}\big)\in R_\ell\Big\}\right].
\end{align*}
The last expectation is exactly $q_\ell(S)$ from \eqref{eq:q-def}, which yields~\eqref{eq:vw-linear}. \QED
\end{proof}

Note that Proposition~\ref{prop:vw-linear} is a conditional statement: the reach probabilities $q_\ell(S)$ are treated as fixed quantities of a fixed tree. Equivalently, the uncertainty considered in this paper is propagated through the leaf outputs $w_\ell$, not through the routing probabilities used to re-estimate $q_\ell(S)$. This distinction is important for the uncertainty model introduced later. 

\subsubsection{Linearity of the attribution value \(\phi(w)\)}
Proposition~\ref{prop:vw-linear} implies that any linear attribution functional applied to the game $v_w$ is itself linear
in $w$. We state this explicitly for SHAP values.

\begin{proposition}\label{prop:shap-linear}
For each feature $i\in\U$, the SHAP value is linear in the leaf-value vector $w$:
\begin{equation}\label{eq:shap-leaf-linear}
\phi_i(w)=\sum_{\ell\in\mathcal{L}} a_{i,\ell}\,w_\ell=\langle a_i,w\rangle,
\end{equation}
where the coefficient vector $a_i=(a_{i,\ell})_{\ell\in\mathcal L}\in\mathbb R^{|\mathcal L|}$ is given by
\begin{equation}\label{eq:a-coeff}
a_{i,\ell}
=
\sum_{S\subseteq \U\setminus\{i\}}
\frac{|S|!\,(n-|S|-1)!}{n!}
\Big(q_\ell(S\cup\{i\})-q_\ell(S)\Big).
\end{equation}
In particular, $a_i$ depends only on the tree structure, the explained instance, the background distribution and the missing-feature semantics, and not on the leaf values $w$. Throughout the paper, we use SHAP value to refer to the Shapley value of the feature-attribution game induced by the value function \(v_w\) in (\ref{eq:interventional-game-bg}).
\end{proposition}

\begin{proof}
By Proposition~\ref{prop:vw-linear},
\[
v_w(S\cup\{i\})-v_w(S)
=
\sum_{\ell\in\mathcal L}
\big(q_\ell(S\cup\{i\})-q_\ell(S)\big)\,w_\ell.
\]
Substituting this expression into the Shapley formula \eqref{eq:shapley-bg} and exchanging the finite sums yields \(\phi_i(w) = \)
\[
\sum_{\ell\in\mathcal L}
\left[
\sum_{S\subseteq \U\setminus\{i\}}
\frac{|S|!\,(n-|S|-1)!}{n!}
\big(q_\ell(S\cup\{i\})-q_\ell(S)\big)
\right] w_\ell,
\]
which is exactly \eqref{eq:shap-leaf-linear}--\eqref{eq:a-coeff}. \QED
\end{proof}



\begin{remark}
Banzhaf value is also a linear combination of the marginal differences \(v_w(S\cup\{i\})-v_w(S)\), but with uniform weights. Thus, all our subsequent results apply to Banzhaf indices by replacing the Shapley coefficient \(a_{i,\ell}\) with the Banzhaf coefficient 
\begin{equation}
    b_{i,\ell} = \frac{1}{2^{n-1}}\sum_{S\subseteq \U\setminus \{i\}}\Big(q_\ell(S\cup\{i\})-q_\ell(S)\Big) .
\end{equation}
For this reason, we will focus on SHAP from now on. 
\end{remark}

The reach probabilities $q_\ell(S)$ are fixed once the tree structure, explained instance, background distribution, and missing-feature semantics are fixed. Therefore, once $(a_i)_{i\in\U}$ are computed, SHAP values can only vary according to parameters $w$, reducing the analysis to linear forms. This leaf-linear representation is the key ingredient behind the interval-valued robustness analysis developed in the next section.

\section{Interval-valued feature attribution under leaf uncertainty}\label{sec:leaf-perturbation}
We now study how to introduce uncertainty in the leaf parameters $w_\ell$, the target-class probabilities. Rather than considering arbitrary variations, we will connect those to virtual samples: this has the benefit that it is easy for the user to specify a number of additional virtual samples of unknown class and features, and see how SHAP values $\phi_i$ vary from~it.

\subsection{From virtual samples to leaf uncertainty}\label{subsec:idm-leaf-uncertainty}
Since the tree structure is fixed, uncertainty is introduced only through the leaf-output vector
\(w=(w_\ell)_{\ell\in\mathcal L}\). This can model retraining variability or moderate distribution shift affecting leaf estimates. For a target-class probability, let \(N_\ell\) be the cover of leaf \(\ell\), and let \(n_\ell\) be the number of samples of the target class in that leaf. The nominal leaf output is therefore \(w_\ell=\frac{n_\ell}{N_\ell}.\) Depending on the value of $N_\ell$, a small variation of $n_\ell$ can have a deep impact on the value of $w_\ell$. 


For this reason, it may be useful to check how sensitive our explanations are (taking the form of SHAP values) if we add a small number of instances to the leaf. A simple idea to do that is to add to leaf $\ell$ a number \(m_\ell\) of unclassified instances. This gives rise to the following possible interval of values $w_\ell$: 
\[
[\underline{w}_\ell,\overline{w}_\ell]=\left[\frac{n_\ell}{N_\ell+m_\ell},
\frac{n_\ell+m_\ell}{N_\ell+m_\ell}
\right],
\]
where the lower bound is obtained by considering that all $m_\ell$ instances do not belong to the target class, and the upper by considering that they all belong to the target class. These bounds can either be interpreted in a frequentist way, in which case they are derived from imprecise observations~\cite{dempster2008upper}, or in a Bayesian way, in which case they would correspond to the Imprecise Dirichlet Model (IDM) \cite{10.1111/j.2517-6161.1996.tb02065.x}. 

Note that it is difficult for a user to specify a meaningful number $m_\ell$ of unknown virtual instances per leaf. This is why we will consider that the user specifies a unique number \(s\) of virtual instances to be distributed among the tree leaves. Given this, an allocation of virtual instances will be a vector
\[
m=(m_\ell)_{\ell\in\mathcal L}\in\mathcal M(s)
:=
\left\{
m\in\mathbb N^{|\mathcal L|}
:
\sum_{\ell\in\mathcal L}m_\ell=s
\right\}.
\]
Note that the uncertainty level \(s\) has a more direct statistical interpretation than a uniform per-leaf allocation used in \cite{pmlr-v147-zhang21a}: it measures how many imprecise observations the explanation is required to tolerate, without assuming where these observations would fall in the tree. As we will see next, this makes the optimisation problems associated with the sensitivity analysis more complex, but we think it is much more understandable from the user's perspective, which is essential in explainability. 

\begin{example}\label{ex:idm}
We use a toy depth-2 tree with three split features and four leaves as a running example. Consider the allocation \(m=(1,2,1,2)\), corresponding to 6 unannotated instances. The resulting IDM intervals are reported in Table~\ref{tab:ex1}.
\begin{table}[htbp]
\vspace{-2em}
\centering
\caption{IDM leaf-output intervals.}
\label{tab:ex1}
\footnotesize
\renewcommand{\arraystretch}{1.15}
\begin{tabular}{c|ccccc}
\hline
leaf & \(n_\ell\) & \(N_\ell\) & \(w_\ell\) & \(m_\ell\)
& \([\underline{w}_\ell,\overline{w}_\ell]\) \\
\hline
\(\ell_1\) & 20 & 60 & \(1/3\) & 1
& \(\left[\frac{20}{61},\frac{21}{61}\right]\) \\
\(\ell_2\) & 20 & 50 & \(2/5\) & 2
& \(\left[\frac{20}{52},\frac{22}{52}\right]\) \\
\(\ell_3\) & 10 & 30 & \(1/3\) & 1
& \(\left[\frac{10}{31},\frac{11}{31}\right]\) \\
\(\ell_4\) & 30 & 50 & \(3/5\) & 2
& \(\left[\frac{30}{52},\frac{32}{52}\right]\) \\
\hline
\end{tabular}
\vspace{-2em}
\end{table}

\end{example}

\subsection{Interval-valued attributions}
In this section, we study how we can integrate the virtual samples in order to derive interval-valued attributions of each feature. We will consider two settings: one where we have no information about the distribution of virtual samples (unannotated samples) into the leaf, which we will call \emph{unknown allocation}, and one where we have an idea about how the virtual sample population is distributed within the leaves, which we will call \emph{stochastic allocation}. 

Before detailing each approach, recall that by Section~\ref{sec:leaf-linearity}, for each feature $i\in\U$ there exists a coefficient vector
$ a_i \in\R^{|\mathcal L|}$ such that
\[
\phi_i(w)=\langle a_i,w\rangle
\qquad\text{for all } w\in\R^{|\mathcal L|},
\]
meaning that $\phi_i$ is a linear function of $w$. 

\subsubsection{unknown allocation (pessimistic)}

We will consider the problem of deriving upper and lower bounds of $\phi_i$ given that any allocation can happen. Let
\begin{equation*}
    \mathrm{ran}(w) = \bigtimes_{\ell \in \mathcal{L}} [\underline{w}_\ell,\overline{w}_\ell].
\end{equation*}
Formally, we will aim at solving the following problems:
\begin{gather}
    \underline{\phi}_i^U=\inf_{m \in \mathcal{M}(s)} \inf_{w \in \mathrm{ran}(w)} \langle a_i,w\rangle \label{eq:unknown_low}, \\
       \overline{\phi}_i^U=\sup_{m \in \mathcal{M}(s)} \sup_{w \in \mathrm{ran}(w)} \langle a_i,w\rangle \label{eq:unknown_up}. 
\end{gather}
We refer to this as the \emph{pessimistic} principle: since the allocation is unknown, the interval is enlarged by considering the most adverse allocation for each endpoint. We will focus on solving the lower bound, and will show that the derived solution is analogous to the upper bound. Let us first focus on the inner optimization problem, which is straightforward to solve given its linearity
\begin{equation}\label{eq:inner_lower_shap}
\underline{\phi}_i^{inn}=\inf_{w \in \mathrm{ran}(w)} \langle a_i,w\rangle =\sum_{\ell\in\mathcal L}
\begin{cases}
a_{i,\ell}\,\underline w_\ell, & a_{i,\ell}\ge 0,\\
a_{i,\ell}\,\overline w_\ell, & a_{i,\ell}<0,
\end{cases}\end{equation}
So, the problem becomes
$$\underline{\phi}_i^U=\inf_{m \in \mathcal{M}(s)} \underline{\phi}_i^{inn}$$
and we now have to find the allocation $m$ for which the minimum will be reached. In order to do that, we will first show that Equation~\eqref{eq:inner_lower_shap} can be reformulated more conveniently.
\begin{proposition}\label{prop:inner_rewriting}
    Equation~\eqref{eq:inner_lower_shap} can be reformulated as
    $$\underline{\phi}_i^{inn}=\phi_i(w) - C^{\mathrm{wc}}_{\alpha^-_i}(m)$$
    where
\begin{align}\label{def:allocation_correction}
C^{\mathrm{wc}}_{\alpha^-_i}
&:=
\sum_{\ell\in\mathcal L}
\alpha^-_{i,\ell}\,
\frac{m_\ell}{N_\ell+m_\ell}, \\
\alpha^-_{i,\ell}
&:=
[a_{i,\ell}]_+\,w_\ell
+
[-a_{i,\ell}]_+\,(1-w_\ell),
\end{align}
and $[u]_+:=\max\{u,0\}$, is a correction term depending on~$m$. 
\end{proposition}

\begin{proof}
    First note that we can rewrite 
    $$\underline w_\ell= w_\ell- w_\ell\frac{m_\ell}{N_\ell+m_\ell}.$$
    Indeed, we have \(w_\ell- w_\ell\frac{m_\ell}{N_\ell+m_\ell}\)
    \begin{align*}
         & = \frac{n_\ell}{N_\ell}-\frac{n_\ell}{N_\ell}\frac{m_\ell}{N_\ell+m_\ell} 
         =  \frac{n_\ell (N_\ell + m_\ell) - n_\ell m_\ell}{N_\ell(N_\ell+m_\ell)} \\ &=\frac{N_\ell n_\ell}{N_\ell(N_\ell + m_\ell)} =\frac{n_\ell}{N_\ell + m_\ell} = \underline w_\ell. 
    \end{align*}
    Likewise, we can rewrite 
    $$\overline w_\ell =  w_\ell+(1- w_\ell)\frac{m_\ell}{N_\ell+m_\ell}.$$
    Plugging those rewrittings in Eq.~\eqref{eq:inner_lower_shap}, we get 
\begin{align}
        \underline{\phi}^{inn}_i &=\sum_{\ell \in \mathcal{L}} w_\ell a_{i,\ell}+ \begin{cases}
- a_{i,\ell} w_\ell \frac{m_\ell}{N_\ell+m_\ell}, & a_{i,\ell}\ge 0,\\
a_{i,\ell}(1- w_\ell)\frac{m_\ell}{N_\ell+m_\ell}, & a_{i,\ell}<0. 
\end{cases}    \\
& = \phi_i(w)
-
\sum_{\ell\in\mathcal L}
\alpha^-_{i,\ell}\,
\frac{m_\ell}{N_\ell+m_\ell} \\
&=\phi_i(w) - C^{\mathrm{wc}}_{\alpha^-_i}(m).
\end{align}
\QED
\end{proof}

So, now, all we have to do is to solve 
\begin{align*}
\underline{\phi}_i^U& =\inf_{m \in \mathcal{M}(s)} \phi_i(w) - C^{\mathrm{wc}}_{\alpha^-_i}(m) = \phi_i(w) - \sup_{m \in \mathcal{M}(s)} C^{\mathrm{wc}}_{\alpha^-_i}(m) \\
&= \phi_i(w) - C^{\mathrm{wc}}_{\alpha^-_i}(s),
\end{align*}
which means we need to find the allocation maximizing $C^{\mathrm{wc}}_{\alpha^-_i}(m)$. We first note that 
\[
C^{\mathrm{wc}}_{\alpha^-_i}(m)= \sum_{\ell\in\mathcal L}
\alpha^-_{i,\ell}\,
\frac{m_\ell}{N_\ell+m_\ell} = \sum_{\ell\in\mathcal L}
\alpha^-_{i,\ell}\ g_l(m_\ell)
\]
is a separable function. Moreover, it is easy to check that each function $g_l$ is concave so that $C^{\mathrm{wc}}_{\alpha^-_i}(m)$ is also concave. Finding the optimal $m$ is thus a resource allocation problem which can be solved using a greedy approach, the steps of which can be summarized in the following algorithm:

\begin{algorithm}
\begin{algorithmic}[1]
\caption{Optimal allocation of virtual samples}
\label{alg:greedy-exact}
\State Start with $m_l=0$ for all $l$
\State Compute for each $l$ the marginal gain: 
    \begin{align}
\Delta_\ell(m_\ell + 1)
& :=
\alpha^-_{i,\ell}
\left(
\frac{m_\ell + 1}{N_\ell+m_\ell + 1}
-
\frac{m_\ell}{N_\ell+m_\ell}
\right) \nonumber \\
& =
\alpha^-_{i,\ell}
\frac{N_\ell}{(N_\ell+m_\ell+1)(N_\ell+m_\ell)};
\end{align}
\State  Assign a new sample to the leaf with the largest gain;
\State Repeat this sequence $s$ times.
\end{algorithmic}
\end{algorithm}

Algorithm~\ref{alg:greedy-exact} runs in \(O(s|\mathcal L|)\) time with a direct implementation. The optimization problem for the upper bound can be treated likewise, since if we define 
\begin{equation}\label{eq:upper_decomposition}
\alpha^+_{i,\ell}
:=
[a_{i,\ell}]_+\,(1-w_\ell)
+
[-a_{i,\ell}]_+\,w_\ell,
\end{equation}

Then the inner supremum can be rewritten as 
\[
\overline\phi_i^{inn}
=
\phi_i(w)
+
\sum_{\ell\in\mathcal L}
\alpha^+_{i,\ell}\,
\frac{m_\ell}{N_\ell+m_\ell}
=\phi_i(w)
+
 C^{\mathrm{wc}}_{\alpha^+}(m)
 \]
which is of a form similar to Proposition~\ref{prop:inner_rewriting}, except that we have $\overline\phi_i^{inn}=\phi_i(w) + C^{\mathrm{wc}}_{\alpha^+_i}(m)$, with the need to find the allocation maximising the second term, meaning we can apply Algorithm~\ref{alg:greedy-exact} to solve it.

In practice, the greedy allocation used to compute the lower correction and the one used to compute the upper correction need not coincide; moreover, both allocations may depend on the feature \(i\). 

\subsubsection{Stochastic allocation (averaging)}

An alternative is to model the allocation probabilistically, in which case, assuming that $\mathbb{P}(m)$ is the probability of a given allocation, we have 
\begin{align}
    \underline{\phi}_i^S=\sum_{m \in \mathcal{M}(s)} \mathbb{P}(m) \inf_{w \in \mathrm{ran}(w)} \langle a_i,w\rangle := \mathbb E_M[\underline\phi^{inn}_i(M)]\label{eq:stoch_low}, \\
   \overline{\phi}_i^S=\sum_{m \in \mathcal{M}(s)} \mathbb{P}(m)  \sup_{w \in \mathrm{ran}(w)} \langle a_i,w\rangle := \mathbb E_M[\overline\phi^{inn}_i(M)]\label{eq:stoch_up}, 
\end{align}
We refer to this as the \emph{averaging} principle: instead of optimizing over all feasible allocations, the attribution endpoints are averaged with respect to the allocation distribution \(\mathbb{P}(m)\).
We will consider the case where 
\(
M=(M_\ell)_{\ell\in\mathcal L}\sim \mathrm{Multinomial}(s,\pi),
\) is a Multinomial and 
where $\pi=(\pi_\ell)_{\ell\in\mathcal L}\in\Delta^{|\mathcal L|-1}$ is the leaf-placement distribution. 
Then each marginal satisfies
\(
M_\ell\sim \mathrm{Binomial}(s,\pi_\ell).
\)

Using our previous decompositions of $\underline{\phi}_{inn}$ and $\overline{\phi}_{inn}$, we~get 
\begin{equation}
\begin{aligned}\label{eq:exp-bounds}
\mathbb E_M[\underline\phi^{inn}_i(M)]
=
\phi_i(w)
-
C^{\mathrm{exp}}_{\alpha_i^-}(s), \\
\mathbb E_M[\overline\phi^{inn}_i(M)]
=
\phi_i(w)
+
C^{\mathrm{exp}}_{\alpha_i^+}(s).
\end{aligned}
\end{equation}
where the expected correction for any nonnegative coefficient vector can be rewritten for
\(\mathrm{sign} \in \{-,+\}\) is
\[
C^{\mathrm{exp}}_{\alpha^{\mathrm{sign}}_i}(s)
=
\mathbb E_M[C^{\mathrm{exp}}_{\alpha^{\mathrm{sign}}_i}(M)]
=
\sum_{\ell\in\mathcal L}
\Bigl(
\mathbb E_M\!\left[
\frac{M_\ell}{N_\ell+M_\ell}
\right]
\alpha^{\mathrm{sign}}_{i,\ell}
\Bigr),
\]
thanks to expectation linearity, and 
where
\begin{equation}\label{eq:binomial_leaf}
\mathbb E_M\!\left[
\frac{M_\ell}{N_\ell+M_\ell}
\right]
=
\sum_{r=0}^s
\frac{r}{N_\ell+r}
\binom{s}{r}
\pi_\ell^r(1-\pi_\ell)^{s-r}.
\end{equation}
due to our multinomial assumption. 

These expected endpoints $\underline{\phi}^S_i,\overline{\phi}^S_i$ characterize average attribution uncertainty under the placement law \(\pi\). They should be distinguished from the worst-case endpoints $\underline{\phi}^U_i,\overline{\phi}^U_i$: the latter optimize over all allocations, whereas the former average over allocations sampled from a specified distribution. Since the correction terms are nonnegative,
the expected interval is no wider than the corresponding worst-case envelope. In practice, one may choose \(\pi\) to be, for example, uniform over leaves, \(\pi_\ell=1/|\mathcal L|\), or proportional to empirical leaf mass, \(\pi_\ell=N_\ell/\sum_j N_j\), depending on the intended interpretation. 

\vspace{0.2em}
\begin{example}\label{ex:allocation}
We continue Example~\ref{ex:idm}. Consider the lower correction for a fixed feature \(i\), and suppose that the coefficients in the correction term (\ref{def:allocation_correction}) are \(\alpha=(0.60,0.45,0.30,0.50).\)
In Table~\ref{tab:allocation-computation-example}, under unknown allocation, the adversary chooses the allocation maximizing \(C^{\mathrm{wc}}_{\alpha^-_i}(m)\). Starting from \(m=(0,0,0,0)\), the greedy rule assigns each new virtual sample to the leaf with the largest current marginal gain \(\Delta_\ell\). Under stochastic allocation, we instead assume \(M\sim\mathrm{Multinomial}(6,\pi)\), with empirical leaf masses
\(
\pi=
\left(
\frac{60}{190},
\frac{50}{190},
\frac{30}{190},
\frac{50}{190}
\right).
\)
As expected, \(C_{\alpha^-_i}^{\mathrm{wc}}(6)\) is greater than~\(C_{\alpha^-_i}^{\mathrm{exp}}(6)\).

\begin{table}[htbp]
\vspace{-1em}
\centering
\caption{Worst-case and stochastic allocation computations.}
\label{tab:allocation-computation-example}
\scriptsize
\renewcommand{\arraystretch}{1.15}
\begin{minipage}{0.48\linewidth}
\centering
\begin{tabular}{c|c|c|c}
\hline
step & leaf & \(\Delta_\ell\) & \(m\) \\
\hline
1 & \(\ell_1\) & 0.00984 & \((1,0,0,0)\)\\
2 & \(\ell_4\) & 0.00980 & \((1,0,0,1)\)\\
3 & \(\ell_3\) & 0.00968 & \((1,0,1,1)\)\\
4 & \(\ell_1\) & 0.00952 & \((2,0,1,1)\)\\
5 & \(\ell_4\) & 0.00943 & \((2,0,1,2)\)\\
6 & \(\ell_1\) & 0.00922 & \((3,0,1,2)\)\\
\hline
\multicolumn{4}{c}{\(C_{\alpha_i^-}^{\mathrm{wc}}(6) \approx 0.0575\)}\\
\end{tabular}
\end{minipage}
\hfill
\begin{minipage}{0.48\linewidth}
\centering
\begin{tabular}{c|c|c}
\hline
leaf & \(\pi_\ell\) & \(\mathbb E\!\left[\frac{M_\ell}{N_\ell+M_\ell}\right]\)\\
\hline
\(\ell_1\) & \(60/190\) & 0.0303\\
\(\ell_2\) & \(50/190\) & 0.0302\\
\(\ell_3\) & \(30/190\) & 0.0298\\
\(\ell_4\) & \(50/190\) & 0.0302\\
\hline
\multicolumn{3}{c}{\(C_{\alpha_i^-}^{\mathrm{exp}}(6) \approx 0.0558\)}\\
\end{tabular}
\end{minipage}
\vspace{-2em}
\end{table}

\end{example}

\subsection{Extensions to tree-ensemble}

We now discuss how the leaf-uncertainty framework extends to tree ensembles. Let $\mathcal{T} = \{T_1, \ldots T_T\}$ be the set of $T$ trees. Let \(
\tau_{\mathcal{T}}=(\ell^1,\ldots,\ell^T)
\) be the set of leaves (one per tree) that cover $\vec{x}$.  The ensemble prediction is

\begin{align}\label{eq:ensemble output}
F_{\mathcal{T},w}(x)=\sum_{t=1}^T \eta_t f_{t,w^{(t)}}(\vec{x})
\end{align}
where \(f_{t,w^{(t)}}(x)\) is the prediction of the \(t\)-th tree, \(w^{t}\) is its leaf-output
vector, \(\eta_t\) is its weight. For a random forest, typically \(\eta_t=1/T\); for gradient-boosted trees such as XGBoost, the weights include the learning-rate scaling.

By linearity of attributions in the model output, the single-tree leaf-linear representation extends directly to ensembles:
\begin{align}
\phi_i(F_{\mathcal{T},w})
&=
\sum_{t=1}^T \eta_t \phi_i(f_{t,w^{(t)}})
=
\sum_{t=1}^T \eta_t \langle a_i^{(t)},w^{(t)}\rangle \nonumber\\
&=
\langle a_i^{\mathrm{ens}},w^{\mathrm{ens}}\rangle, \label{eq:attribution in ensemble}
\end{align}
where \(w^{\mathrm{ens}}\) is the concatenation of all leaf-output vectors and \(a_i^{\mathrm{ens}}\) is the corresponding concatenated coefficient vector. Thus, under fixed tree structures, ensemble attributions remain linear in the collection of leaf outputs.

Let \(
\mathcal L_{\mathcal{T}}=\mathcal L_1\times\cdots\times\mathcal L_{\mathcal T}
\). The ensemble $\mathcal{T}$ can be interpreted as a (possibly very) large tree, whose leaves are non-empty intersections:   

\begin{equation}\label{eq:leaf in ensemble}
   \ell_{\mathcal{T}} = \cap_{t=1}^T \ell_t \neq \emptyset \,, \forall \tau_{\mathcal{T}} \in L_{\mathcal{T}}. 
\end{equation}
For each leaf $\ell_{\mathcal{T}}$, the attribution is defined using \eqref{eq:attribution in ensemble}. Therefore, one can, in principle, convert the ensemble into a tree with the leaf output \eqref{eq:ensemble output} and compute interval-valued attribution following either the pessimistic or averaging principle. However, constructing the tree can be costly.  

We therefore propose to proceed with an approximate approach. When proceeding with the pessimistic approach, we relax the requirement that the leaf $\ell_{\mathcal{T}}$ \eqref{eq:leaf in ensemble} should be non-empty and define the attribution using \eqref{eq:attribution in ensemble} even if the leaf is empty. The greedy algorithm proposed in Algorithm \ref{alg:greedy-exact} can be easily extended to find an allocation on the (possibly empty) leaves of the ensemble that is optimal under the pessimistic principle. Technically, we start from $m_\ell^t = 0$, for any $\ell$ and $t$, and repeatedly assign the next example to elements of a set of leaves $\tau_{\mathcal{T}}$ with the highest marginal gains. At each iteration, this set of leaves can be found efficiently. Thanks to the linearity of the attribution \eqref{eq:attribution in ensemble}, we can look at each tree separately and select the leaf with the highest marginal gain on that tree. Since the marginal gains achieved during the allocation procedure can come from invalid, i.e., empty, leaves of the ensemble, the interval-valued attributions given by this approximate approach are an outer approximation of the interval-valued attributions produced by converting the ensemble into a tree and then proceeding with the pessimistic~principle.

One can, of course, think of an approximate approach to the averaging principle by averaging the ensemble correction over all the (possibly empty) leaves $\ell_{\mathcal{T}}$ of the ensemble. This will be equivalent to averaging the corrections of the leaves of the trees. However, the interval-valued attributions are not necessarily either the attributions provided by the exact approach or their outer approximation. 

\subsection{Discussion}
In many applications, feature attributions are mainly used through rankings rather than through their exact numerical values. It is therefore natural to ask whether such rankings remain stable under leaf uncertainty. A direct comparison of individual attribution intervals can be misleading, since the lower and upper endpoints of different features may be attained under different leaf-output vectors. For ranking stability, it is more appropriate to compare two features under the same leaf vector \(w\). For two features \(i\) and \(j\), define the pairwise attribution margin
\(
\gamma_{ij}(w)
:=
\phi_i(w)-\phi_j(w)
=
\langle a_i-a_j,w\rangle .
\)
Since \(\gamma_{ij}\) is again linear in the leaf-output vector, the same interval-valued analysis applies by replacing the attribution coefficient \(a_i\) with the margin coefficient \(a_i-a_j\). 

Given a nominal ranking \(i_{(1)},\ldots,i_{(n)}\), one can test the stability of the whole ranking by checking the adjacent margins 
\(
\gamma_{i_{(k)} i_{(k+1)}}(w), \; k=1,\ldots,n-1.
\)
If all adjacent lower margins are positive, then the complete ranking is robust at uncertainty level \(s\). If some adjacent comparisons are ambiguous, the method may still certify a coarser tiered ranking. For example, even if features \(1\) and \(2\) cannot be robustly ordered, both may remain certified above feature \(3\), yielding the stable statement
\(
\{1,2\}\succ_s \{3\}.
\)
This tiered interpretation is often sufficient when the goal is to identify a robust set of top features rather than a fully ordered ranking.

This ranking viewpoint also provides a practical interpretation of \(s\). In the original IDM, \(s\) represents prior strength, and Walley \cite{10.1111/j.2517-6161.1996.tb02065.x} recommends small values, typically \(s=1\) or \(s=2\). In our setting, where \(s\) represents a virtual-sample uncertainty budget, ranking stability provides an application-oriented selection criterion: as \(s\) increases, pairwise ranking certificates become harder to maintain, and the robust ordering between two features, or between two groups of features, may eventually fail. For example, at the dataset level, one can fix the nominal top-\(k\) set of each instance and track the fraction of instances for which this set remains robustly top-ranked as \(s\) increases from zero. The largest \(s\) for which this fraction remains above a prescribed threshold then quantifies how much leaf uncertainty the top-\(k\) explanations can tolerate while maintaining the desired level of stability.

Since ranking certification is not the main focus of this work, we leave a detailed study of margin-based ranking stability for single trees and ensembles as future work. In the experiments below, whenever interval comparisons are needed, we use the simpler sufficient criterion
\(
\underline{\phi}_i > \overline{\phi}_j
\)
to certify that feature \(i\) is ranked above feature \(j\).

\section{Experiments}\label{sec:experiments}

We evaluate the proposed interval-valued attribution framework from the viewpoint of a practitioner who uses feature-importance explanations to interpret tree-based predictions. The objective is not only to show how to interpret and robustify the framework, but also to provide empirical evidence for potential applications. The source code is available in the repository \href{https://github.com/Chenrui-ZHU/Interval-valued-SHAP}{GitHub}. We organize the experiments around the following questions:

\begin{enumerate}
    \item For a fixed instance, which nominally important features remain robust under a given uncertainty level, and which ones become fragile? \label{ques:local}
    \item Does the proposed interval-valued attribution agree with empirical attribution variability under bootstrap model retraining?\label{ques:bootstrap}
    \item How does interval-valued attribution vary between confident and ambiguous decision regions?\label{ques:width-vs-decision}
    \item Can the pessimistic principle reduce the apparent importance of uninformative features?\label{ques:debias}
\end{enumerate}

Since our robustness notions are defined for a single explained instance, the primary analysis is local. We first examine individual predictions to show how interval-valued attributions distinguish robust from fragile feature-importance claims. We then aggregate the same instance-level quantities over the test set. These summaries provide a dataset-level view of attribution robustness while preserving the local nature of the proposed interval-valued SHAP. 

\subsection{Experimental protocol}

We first use decision trees for Question~\ref{ques:local}, where the proposed uncertainty model is clearly interpretable. We then consider random forests for the remaining questions,  in order to evaluate how the robustness analysis behaves in ensemble settings. We use four public tabular binary-classification datasets: \emph{Diabetes} (Diab.), \emph{Breast Cancer} (BreC.), \emph{Ionoshpere} (Ions.), \emph{Mortality} (Mort.). The first three datasets are available from \href{https://www.openml.org/search?type=data&sort=runs&status=active}{OpenML}. The \emph{Mortality} dataset is the case study from \cite{lundberg2019explainable}. The summary of the datasets is shown in Table~\ref{tab:dataset-summary}.

\begin{table}[htbp]
\centering
\caption{Summary of datasets}
\label{tab:dataset-summary}
\begin{tabular}{c|cccc}
\hline
 & Diab. & BreC. & Ions. & Mort. \\
\hline
Features  & 8       & 30      & 34          & 79 \\
Instances & 768     & 569     & 351       & 14264 \\
\hline
\end{tabular}
\end{table}


\vspace{-0.3em}

For each dataset, we use a 5-fold cross-validation protocol. 
Model hyperparameters are selected on the training set using 3-fold nested cross-validation, with balanced accuracy as the selection criterion. For both decision trees and random forests, we tune the minimum number of samples per leaf over \(\{5, 10, 20\}\). Random forests use 100 trees. We then compute instance-level SHAP intervals on the test set and report averages over the five folds. Experiments were run on Ubuntu 24.04 using a machine with 80 CPU cores (2 \(\times\) Intel(R) Xeon(R) Gold 6138 CPU @ 2.00GHz).


\subsection{Case study}\label{sec:case_study}
First, we provide a local understanding of our proposed robustness framework to answer Questions~\ref{ques:local}. We present a case study on selected instances to see whether SHAP values and their induced rankings remain stable. Specifically, from a user perspective, the key takeaways are:
(i) \emph{dominance}, where a feature remains clearly more influential than another because their intervals are separated; (ii) \emph{incomparability}, where the relative order of two features is ambiguous because their intervals overlap; and (iii) \emph{robust top-k}, where a group of features dominates all features outside the group even if the ordering inside the group is ambiguous.

We illustrate first how to interpret interval-valued SHAP scores on the \emph{Diabetes} dataset, which predicts whether a patient is diabetic from eight clinical variables: pregnancies (\texttt{preg}), plasma glucose (\texttt{plas}), blood pressure (\texttt{pres}), skin-fold thickness (\texttt{skin}), insulin (\texttt{insu}), BMI (\texttt{mass}), diabetes pedigree (\texttt{pedi}), and age (\texttt{age}). Their direct clinical meaning makes this dataset suitable for interpretation.

We train a decision tree of depth (6), achieving a test accuracy of (0.76). Table~\ref{tab:instances} presents two representative patients, one predicted non-diabetic and one predicted diabetic. Table~\ref{tab:two-patient-intervals} reports their signed nominal SHAP values and the corresponding magnitude-based intervals at (s=2) under the pessimistic and averaging principles.

\begin{table}[t]
\centering
\caption{Representative instances from the Diabetes dataset, their nominal SHAP values, and magnitude-based intervals. \((+)\) and \((-)\) indicate support for and opposition to the diabetic prediction, respectively}
\label{tab:diabetes_case_study}

\begin{subtable}{0.5\textwidth}
\centering
\caption{Representative instances.}
\label{tab:instances}
\begin{tabular}{c|cccccccc}
\toprule
sample & preg & plas & pres & skin & insu & BMI & pedi & age \\
\midrule
1 & 2 & 108 & 62 & 32 & 56 & 25.2 & 0.128 & 21 \\
2 & 5 & 166 & 72 & 19 & 175 & 25.8 & 0.587 & 51 \\
\bottomrule
\end{tabular}
\end{subtable}

\vspace{0.5em}

\begin{subtable}{0.5\textwidth}
\centering
\caption{Nominal SHAP values and magnitude-based intervals.}
\label{tab:two-patient-intervals}
\renewcommand{\arraystretch}{1.05}
\setlength{\tabcolsep}{3pt}
\begin{tabular}{c|c|ccc}
\toprule
sample & feature & nominal & pessimistic & averaging\\
\midrule
\multirow{6}{*}{1}
 & plas & (-) 0.137 & $[0.133,0.141]$ & $[0.135,0.138]$\\
 & BMI  & (-) 0.130 & $[0.122,0.135]$ & $[0.128,0.131]$\\
 & age  & (-) 0.059 & $[0.053,0.073]$ & $[0.057,0.060]$\\
 & pedi & (-) 0.021 & $[0.019,0.022]$ & $[0.021,0.021]$\\
 & preg & (-) 0.004 & $[0.003,0.005]$ & $[0.004,0.004]$\\
 & pres & (+) 0.003 & $[0.002,0.004]$ & $[0.003,0.003]$\\
\midrule
\multirow{6}{*}{2}
 & plas & (+) 0.357 & $[0.334,0.365]$ & $[0.353,0.359]$\\
 & pedi & (+) 0.108 & $[0.099,0.114]$ & $[0.107,0.109]$\\
 & BMI  & (-) 0.093 & $[0.084,0.095]$ & $[0.092,0.094]$\\
 & age  & (+) 0.029 & $[0.027,0.034]$ & $[0.028,0.030]$\\
 & preg & (-) 0.003 & $[0.002,0.003]$ & $[0.003,0.003]$\\
 & pres & (+) 0.001 & $[0.000,0.001]$ & $[0.001,0.001]$\\
\bottomrule
\end{tabular}
\end{subtable}
\end{table}

\paragraph{Patient 1 (predicted non-diabetic)}
The magnitude-based ranking indicate that low plasma glucose, lower BMI, and young age support the non-diabetic prediction. Under the pessimistic principle, the intervals for \texttt{plas} and \texttt{BMI} overlap, so their internal order is ambiguous: either feature may become the most important. By contrast, both remain clearly separated from all remaining features. Hence, while the exact top-1 feature is unstable, the top-2 set is robust, yielding the user-facing conclusion:
\emph{``glucose and BMI are the two dominant drivers of the prediction, while age has a secondary effect.''} Under the averaging principle, the intervals are narrower and the intervals for \texttt{plas} and \texttt{BMI} no longer overlap, preserving their nominal order. This comparison provides a more cautious explanation than the nominal ranking alone, because making a non-risky prediction requires relying on several consistently influential features rather than over-interpreting the exact order between \texttt{plas} and \texttt{BMI}. If \(s\) increases further, additional interval expansion may eventually cause \texttt{age} to overlap with \texttt{BMI}, weakening the robust top-2 interpretation.

\paragraph{Patient 2 (predicted diabetic)}
Here, plasma glucose and pedigree contribute positively to the diabetic prediction, whereas BMI contributes negatively and therefore acts as a countervailing factor. Unlike Patient~1, the leading feature is unambiguous: \texttt{plas} is robustly top-1, as its interval is well separated from all others. Moreover, the intervals for \texttt{pedi} and \texttt{BMI} remain separated under both principles, although the gap is small. The averaging principle produces narrower intervals but leads to the same ranking conclusion. This supports the interpretation: \emph{``the prediction is dominated by glucose, while the ranking of the secondary drivers might be fragile.''} This is also reasonable given the patient's high plasma glucose value, together with a relatively high pedigree value and a lower~BMI.

\subsection{Bootstrap-based validation of interval-valued attributions}\label{subsec:bootstrap-validation}

We now move from individual instances to a dataset-level analysis. For Question~\ref{ques:bootstrap}, we examine how bootstrap retraining variability correlates with the interval-valued attributions. More specifically, we ask whether larger attribution intervals induced by our method also tend to exhibit a larger variability under bootstrap retraining. As illustrated in the preceding case study and expected from the theoretical formulation, the \emph{pessimistic} principle generally produces wider intervals than the \emph{averaging} one, thereby providing a more conservative assessment of attribution uncertainty. For the sake of simplicity, we therefore focus on \emph{pessimistic} in the following experiments.

In this experiment, the number of bootstrap replicates is fixed to \(\mathcal B=100\) for single trees (T) and for random forests (RF). For each fold, we first train the model on the whole training set and compute, for each test instance \(\vec{x}\), a vector
\[
    \mathrm{len}(\vec{x}) = (\mathrm{len}_1, \ldots, \mathrm{len}_n),
\]
where \(\mathrm{len}_i := \overline{\phi}_i - \underline{\phi}_i\) is the \emph{proposed attribution interval length}. Inside each fold, for each bootstrap replicate \(b \in \{1,\cdots, \mathcal B\}\), we retrain the model on a bagging-style bootstrap sample (sampling observations uniformly with replacement) of the training set and compute the nominal attribution \(\phi^{(b)}_i\) for the same instance. A bootstrap-based vector is
\[
    \mathrm{len}^{\mathrm{boot}}(\vec{x}) = (\mathrm{len}^{\mathrm{boot}}_1, \ldots, \mathrm{len}^{\mathrm{boot}}_n),
\]
where $\mathrm{len}^{\mathrm{boot}}_i := \max_{1 \leq b \leq \mathcal B}{\phi}^{(b)}_i - \min_{1 \leq b \leq \mathcal B}{{\phi}^{(b)}_i}$ is the \emph{bootstrap-based attribution interval length} by the empirical min-max. Then we compute the Spearman correlation \cite{KING201923} of the two vectors given by the two approaches. This correlation evaluates whether the two approaches rank attribution length similarly according to attribution variability. The reported values are obtained by averaging the instance-level correlations over all test instances and folds. 

The idea is to see if the proposed interval-valued attributions are aligned with the empirical variability produced by the bootstrap retraining. 
The results given in Table~\ref{tab:bootstrap-correlation} confirm this relation. Across \(s\in\{1,2,3\}\), the correlations varied by at most \(0.01\) over all dataset–model combinations, showing that the observed association is not specific to a certain s. So we only report the results for \(s=2\).
The correlations are moderately high for single trees and high for random forests. Since the proposed intervals avoid repeated model retraining, they provide a computationally convenient proxy for assessing bootstrap attribution sensitivity.

\begin{table}[htbp]
\centering

\begin{minipage}[t]{0.48\linewidth}
\centering
\caption{Correlation between SHAP interval lengths and bootstrap variability.}
\label{tab:bootstrap-correlation}
\begin{tabular}{ccc}
\toprule
Dataset & Model & Spearman \\
\midrule
Diab. & T  & $0.57$ \\
Diab. & RF & $0.83$ \\
BreC. & T  & $0.36$ \\
BreC. & RF & $0.92$ \\
Ions. & T  & $0.40$ \\
Ions. & RF & $0.80$ \\
Mort. & T  & $0.55$ \\
Mort. & RF & $0.90$ \\
\bottomrule
\end{tabular}
\end{minipage}
\hfill
\begin{minipage}[t]{0.48\linewidth}
\centering
\caption{Correlation between the margin-based certainty score
\(\textit{mar}(x)\) and pessimistic SHAP interval length.}
\label{tab:certainty-width-correlation}
\begin{tabular}{cc}
\toprule
Dataset & Spearman \\
\midrule
Diab. & $-0.74$ \\
BreC. & $-0.46$ \\
Ions. & $-0.65$ \\
Mort. & $-0.72$ \\
\bottomrule
\end{tabular}
\end{minipage}

\vspace{-1em}
\end{table}

\subsection{Attribution in certain and ambiguous regions}
To answer Question~\ref{ques:width-vs-decision}, we study whether attribution uncertainty is related to the model's local prediction certainty. Here, we use a simple operational proxy based on the predicted class probabilities. For an instance \(\vec{x}\), let \(\hat p_k(\vec{x})\) denote the predicted probability of class \(k\), and let \(\hat p_{(1)}(\vec{x})\) and \(\hat p_{(2)}(\vec{x})\) be the largest and second largest predicted probabilities.  We define the margin score
\(\textit{mar}(x) := \hat p_{(1)}(\vec{x})-\hat p_{(2)}(\vec{x})\). For binary classification, this reduces to \(\textit{mar}(\vec{x})=|\hat p_1(\vec{x})-\hat p_0(\vec{x})|=|2\hat p_1(\vec{x})-1|\). A large value of \(\textit{mar}(\vec{x})\) indicates that the prediction is confident, whereas a small value indicates that the instance lies in a more ambiguous region of the model. A single decision tree partitions the input space into discrete leaves and produces piecewise-constant class probabilities. As a result, its confidence scores are often too coarse to define meaningful certain and ambiguous regions based on the class-probability margin. In contrast, a random forest averages the outputs of many trees and typically yields a smoother range of confidence scores. We therefore do not use the same margin-based proxy for both model classes, and focus this analysis on random forests. Although this remains only a confidence-based proxy, it is natural in our setting: the confidence score is directly determined by leaf outputs in each tree, and by averaged leaf outputs in the forest.

For each instance \(\vec{x}\), we compare this margin score with the average attribution interval length across all features. We then compute the correlation between \(\textit{mar}(\vec{x})\) and the length over the test instances. Consistently with Section~\ref{subsec:bootstrap-validation}, we use \(s=2\) for the reported results. Table~\ref{tab:certainty-width-correlation} shows a clear negative association for random forests: instances with smaller margin scores tend to have wider attribution intervals, which is expected as predictions in more ambiguous regions are also more fragile from the viewpoint of feature attribution. 


\subsection{Application: debiasing uninformative features}




Tree-based feature-importance measures are known to favor high-cardinality variables, especially in finite samples. The reason is that such variables offer more candidate splits, and therefore, have more opportunities to produce spurious impurity reductions or unstable local contributions. This effect has been studied for MDI, where \cite{li2019debiased} show that noisy features can receive non-zero importance, especially in deep trees with small leaves, and propose \(\mathrm{MDI}^{\mathrm{oob}}\), which evaluates importance using out-of-bag samples. More recently, \cite{loecher2024debiasing} show that global SHAP scores can suffer from a similar bias and propose \(\mathrm{SHAP}^{\mathrm{Shrunk}}\), a sample-splitting correction that shrinks in-sample SHAP values according to their agreement with out-of-sample SHAP values. They also propose a smoothed \(\mathrm{SHAP}^{\mathrm{oob}}\) based on in-bag/out-of-bag samples. For Question~\ref{ques:debias}, we study whether the proposed interval-valued attribution can be used as a conservative debiasing proxy for reducing the apparent importance of uninformative features.

We follow the synthetic protocol of \cite{loecher2024debiasing}. The dataset contains five features \(X_i\) with different cardinalities: \(X_1\) is continuous and normally distributed while \(X_2,\cdots,X_5\) are categorical features with 2,4,10,20 categories, respectively. In the setting, only \(X_2\) is informative, with \(P(Y=1|X_2=1) = 0.5-\rho, \; P(Y=1|X_2=2) = 0.5+\rho\), where \(\rho \in [0,0.5]\) controls the signal strength. A higher $\rho$ indicates a higher importance level of $X_2$ on the values of the class variable $Y$. The remaining features are generated independently of the values of $Y$ and are considered uninformative. For each value of \(\rho\), we generate the prediction \(Y\), split the data into the train/test set with a 75\%/25\% ratio, and train the model on the training set. To isolate the effect of \(\rho\), we keep the feature values and the train/test split fixed across different values of \(\rho\), regenerate only \(Y\), and retrain the model. For our method, we fix \(s=2\).



\subsubsection{Global study}
Motivated by the same phenomenon observed in \cite{loecher2024debiasing}, we consider a global feature-level analysis. Instead of recomputing SHAP values on held-out or out-of-bag samples, we use the lower attribution magnitude induced by the pessimistic principle. Specifically, for feature \(i\), we define
\[ 
\mathrm{SHAP}_i^{\mathrm{Pess}}(s) = \frac{1}{J} \sum_{j=1}^J \inf_{\phi\in[\underline{\phi}^{U}_i(x_j;s),\overline{\phi}^{U}_i(x_j;s)]} |\phi|, 
\]
where \(J\) is the number of explained test instances. In other words, \(\mathrm{SHAP}^{\mathrm{Pess}}\) penalizes features whose apparent importance is not stable. It should not be interpreted as an out-of-sample validation score, since it does not directly test whether the attribution transfers to unseen samples. Rather, it provides a conservative proxy for attribution stability. Since high-cardinality variables can create small or unstable leaves, this proxy can capture part of the same finite-sample fragility targeted by sample-splitting corrections. Using the same signal-strength protocol, we compute nominal SHAP, \(\mathrm{SHAP}^{\mathrm{Pess}}\), and \(\mathrm{SHAP}^{\mathrm{Shrunk}}\) for each fold and average the results over folds. The results are shown in Figure~\ref{fig:debias}. We observe that \(\mathrm{SHAP}^{\mathrm{Pess}}\) reduces the importance assigned to uninformative high-cardinality features.

\begin{figure}[htbp]
    \centering
        \centering
        \includegraphics[width=\linewidth]{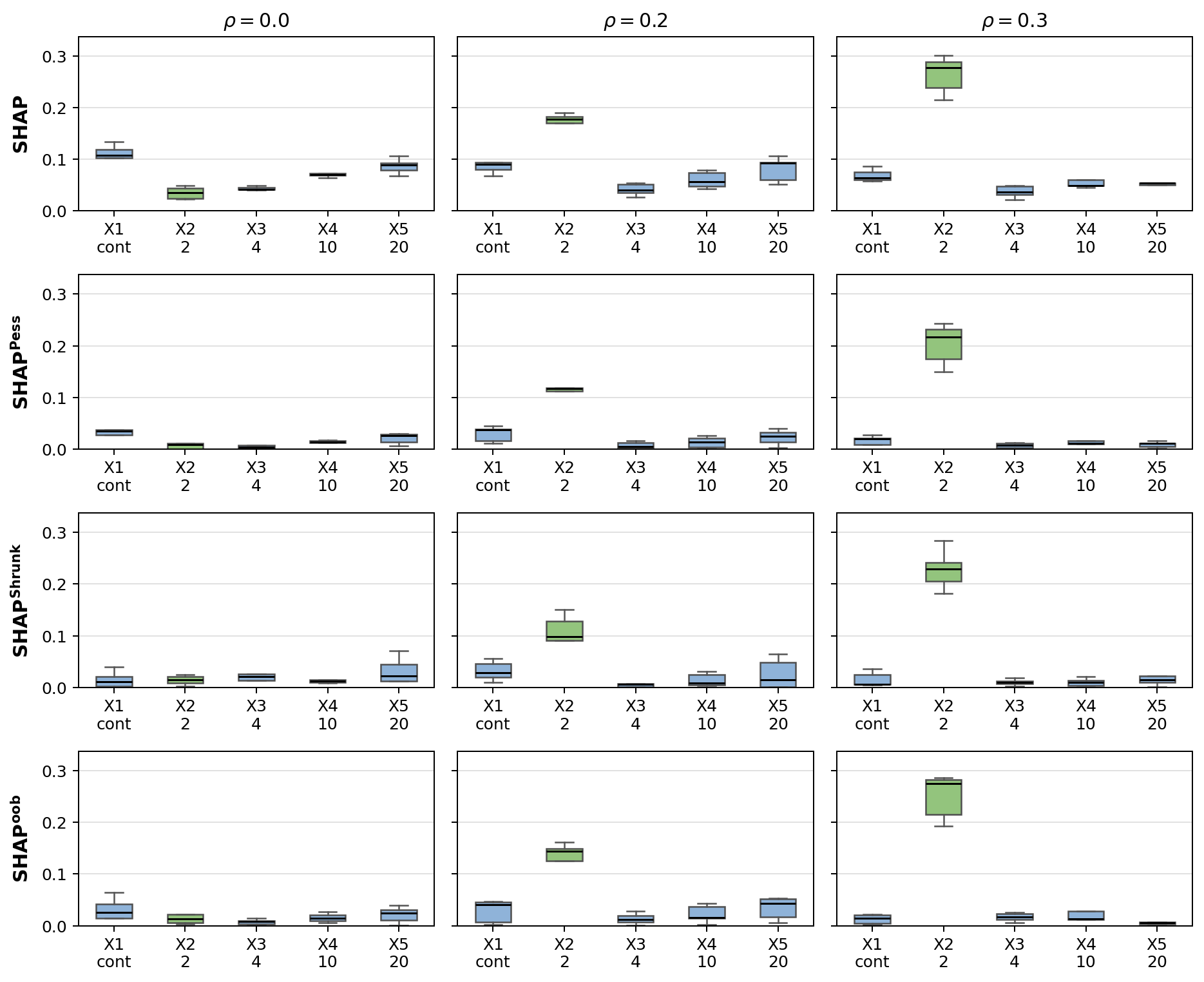}
        \label{fig:synthetic-data-debias}
    \caption{Comparison of (row) nominal SHAP, \(\mathrm{SHAP}^{\mathrm{Pess}}\), \(\mathrm{SHAP}^{\mathrm{Shrunk}}\), and \(\mathrm{SHAP}^{\mathrm{oob}}\), where only \(X_2\) is informative, for different signal strengths (column) \(\rho\in[0,0.3]\).}
    \label{fig:debias}
\end{figure}

The goal of this experiment is not to claim that the pessimistic principle fully removes all sources of high-cardinality bias. Rather, we ask whether the proposed intervals are sufficient to identify features whose importance is fragile. The results suggest that interval-valued SHAP can provide a tractable and interpretable debiasing proxy without retraining the model or recomputing SHAP on additional data splits.

\subsubsection{Feature ranking}
To further evaluate the global feature rankings as a classification task, we follow the feature-ranking setting of \cite{li2019debiased,loecher2024debiasing}, where they treat feature selection as a binary ranking problem since the informative features are known by construction: informative features are assigned label \(1\), and uninformative features are assigned label \(0\). This evaluation is useful in synthetic experiments because the informative set is controlled. However, in practical feature selection, declaring a feature completely uninformative is often too strong: a feature with small importance may still have a non-negligible contribution under the model mechanism. So we define a similar safe SHAP as
\[ 
\mathrm{SHAP}_i^{\mathrm{Safe}}(s) = \frac{1}{J} \sum_{j=1}^J \sup_{\phi\in[\underline{\phi}^{U}_i(x_j;s),\overline{\phi}^{U}_i(x_j;s)]} |\phi|, 
\]

In the experiment, we generate \(n=2000\) observations with \(50\) discrete features. The \(k\)-th feature \(X_k\) takes values in \(\{0,1,\ldots,k\}\), and therefore has \(k+1\) distinct categories. We randomly select a set \(\mathcal S\subseteq\{1,\ldots,10\}\) of five informative features, while the remaining features are uninformative. All features are generated independently. The binary response is then sampled according to
\[
P(Y=1\mid X)
=
\operatorname{Logistic}
\left(
\frac{2}{5}\sum_{k\in\mathcal S}\frac{X_k}{k}-1
\right).
\]
For each SHAP score, we compute the AUC for separating these two groups and then report the averaged AUC score over 5 repetitions. Equivalently, the AUC estimates the probability that a randomly chosen informative feature receives a larger importance score than a randomly chosen uninformative feature. Higher AUC therefore indicates better separation between true signal features and noisy features. The results are reported in Table~\ref{tab:debias-auc}. Overall, both \(\mathrm{SHAP}^{\mathrm{Pess}}\) and \(\mathrm{SHAP}^{\mathrm{Safe}}\) achieve competitive performance compared with other methods. The AUC of \(\mathrm{SHAP}^{\mathrm{Safe}}\) is slightly lower than that of \(\mathrm{SHAP}^{\mathrm{Pess}}\), which is expected. Indeed, \(\mathrm{SHAP}^{\mathrm{Safe}}\) has a different objective: instead of aggressively suppressing uncertain features, it is designed to retain features that may still be important under admissible perturbations. In this sense, it provides a safer criterion for feature selection, especially when removing a potentially informative feature would be undesirable.

\begin{table}[t]
\centering
\caption{Average AUC scores for separating informative from uninformative features.}
\label{tab:debias-auc}
\begin{tabular}{ccccc}
\toprule
SHAP & \(\mathrm{SHAP}^{\mathrm{Safe}}\) & \(\mathrm{SHAP}^{\mathrm{Pess}}\) & \(\mathrm{SHAP}^{\mathrm{Shrunk}}\) & \(\mathrm{SHAP}^{\mathrm{oob}}\)\\
\midrule
0.82 & 0.81 & 0.83 & 0.79 & 0.85 \\
\bottomrule
\end{tabular}
\vspace{-1em}
\end{table}

\section{Conclusion}
\label{sec:conclusion}

In this paper, we propose an IDM-based method to analyze the robustness of SHAP values in tree-based models. It quantifies interval-valued SHAP when a few virtual samples are introduced into the tree leaves. These intervals, defined under the pessimistic and averaging principles, can be computed efficiently. We present case studies and experiments illustrating their behavior and relationship with the empirical variability produced by bootstrap retraining. We also show how our method can better estimate the importance of uninformative features.

As future work, exact computation under both principles could be extended to tree ensembles. One approach is to translate an ensemble into a potentially large single tree and apply the exact solutions presented in Section \ref{sec:leaf-perturbation}. However, constructing this tree and handling its many leaves may be costly. Alternatively, the optimal allocation problem could be formulated as an integer nonlinear program. The case \(s=1\) may be handled efficiently using heuristic algorithms, whereas general values of \(s\) and large ensembles remain challenging.

\section*{Acknowledgments}

This work was partially supported by the Junior Professor Chair in Trustworthy AI \href{https://www.utc.fr/wp-content/uploads/sites/28/2022/07/gi-heudiasyc-cpj-fopc-0601223d-4156.pdf}{(Ref. ANR-R311CHD)}.

\bibliographystyle{elsarticle-num} 
\bibliography{refs.bib}



\vspace{12pt}

\end{document}